\documentclass[10pt,journal,compsoc,twoside]{IEEEtran}
\usepackage[nocompress]{cite}

\usepackage[numbers]{natbib}
\usepackage{graphicx}

\usepackage{bm}
\usepackage{setspace}
\usepackage{amsmath,amssymb,amsfonts,amsthm}
\usepackage{booktabs}

\usepackage{textcomp}
\usepackage{xcolor}
\usepackage[font=footnotesize]{subfig}
\usepackage{multirow}
\usepackage{float}

\newtheorem{theorem}{Theorem}[section]

\newtheorem{proposition}[theorem]{Proposition}

\theoremstyle{remark}
\newtheorem*{remark}{Remark}

\usepackage{hyperref}
\hypersetup{
	colorlinks	= true,
	breaklinks	= true
}

\begin{document}

\title{GraphAIR: Graph Representation Learning with Neighborhood Aggregation and Interaction}

\author{Fenyu Hu, Yanqiao Zhu, Shu Wu, Weiran Huang, Liang Wang, and Tieniu Tan%
\IEEEcompsocitemizethanks{\IEEEcompsocthanksitem F. Hu, Y. Zhu, S. Wu, L. Wang, and T. Tan are with Center for Research on Intelligent Perception and Computing, Institute of Automation, Chinese Academy of Sciences and also with School of Artificial Intelligence, University of Chinese Academy of Sciences.\protect\\
E-mail: \{fenyu.hu, yanqiao.zhu\}@cripac.ia.ac.cn, \{shu.wu, wangliang, tnt\}@nlpr.ia.ac.cn
\IEEEcompsocthanksitem W. Huang is with The Chinese University of Hong Kong.\protect\\
E-mail: huangweiran1998@outlook.com}
\thanks{F. Hu and Y. Zhu contributed equally to this research. The work is done during W. Huang's internship at CRIPAC, CASIA.}}

\markboth{Pattern Recognition}%
{F. Hu and Y. Zhu \MakeLowercase{\textit{et al.}}: GraphAIR: Graph Representation Learning with Neighborhood Aggregation and Interaction}

\IEEEtitleabstractindextext{%
\begin{abstract}
Graph representation learning is of paramount importance for a variety of graph analytical tasks, ranging from node classification to community detection. Recently, graph convolutional networks (GCNs) have been successfully applied for graph representation learning. These GCNs generate node representation by aggregating features from the neighborhoods, which follows the ``neighborhood aggregation'' scheme. In spite of having achieved promising performance on various tasks, existing GCN-based models have difficulty in well capturing complicated non-linearity of graph data. In this paper, we first theoretically prove that coefficients of the neighborhood interacting terms are relatively small in current models, which explains why GCNs barely outperforms linear models. Then, in order to better capture the complicated non-linearity of graph data, we present a novel GraphAIR framework which models the neighborhood interaction in addition to neighborhood aggregation. Comprehensive experiments conducted on benchmark tasks including node classification and link prediction using public datasets demonstrate the effectiveness of the proposed method.
\end{abstract}

}

\maketitle

\IEEEdisplaynontitleabstractindextext
\IEEEpeerreviewmaketitle

\IEEEraisesectionheading{\section{Introduction}}
\IEEEPARstart{G}{raph} representation learning aims to transform nodes on the graph into low-dimensional dense vectors whilst still preserving the attribute features of nodes and structure features of graphs.
In recent years, there has been a surge of research interest in utilizing neural networks to handle graph-structured data. Among them, graph convolutional networks (GCNs) have been shown effective in graph representation learning. They can model complex attribute features and structure features of graphs and achieve the state-of-the-art performance on various tasks. The core of graph convolution is that nodes learn their representations by aggregating features from their neighbors, i.e. the ``neighborhood aggregation'' scheme. Recently, some graph convolutional models, which primarily differ in the neighborhood aggregation strategies, have been proposed \cite{Kipf:2016tc,Velickovic:2018we,Zhuang:2018ks,Xu:2018vn}. For example, GCN \cite{Kipf:2016tc} can be seen as the approximation of aggregation on the first-order neighbors; GraphSAGE \cite{Hamilton:2017tp} designs several aggregators for inductive learning, where unlabeled data does not appear in the training process; GAT \cite{Velickovic:2018we} introduces the attention mechanism to model influence of neighbors with learnable parameters.

From a historical perspective, machine learning research has gone through a long process of development, with one clear trend from simple and linear models to complex and non-linear models. For example, limitations of the linear support vector machine (SVM) motivated the development of non-linear and more expressive kernel-based SVM classifiers \cite{Boser:1992uo}.
Besides, similar trends can be observed in the realm of image processing as real-world data distribution is usually rather complex. For example, simple and linear image filters \cite{Harris:1988tk} are gradually superseded by deep non-linear convolutional neural networks (CNNs) \cite{LeCun:1989wh}.
Driven by the significance of modeling complex and non-linear distributions of data, a question arises: \textit{are existing GCNs capable enough to model the complex and non-linear distributions of graphs?} We find that most previous graph convolutional models (e.g., GCN and GAT) are usually shallow with only one or two non-linear activation function layers, which may restrict the model from well capturing the complicated non-linearity of graph data.

In this paper, we first theoretically prove that the effect of non-linear activation functions in GCNs is to introduce the interaction terms of neighborhood features. We then show that coefficients of the neighborhood interacting terms are relatively small in current GCN-based models. To this end, we present a general framework named GraphAIR (\underline{A}ggregation and \underline{I}nte\underline{R}action) \footnote{Code is publicly available at \url{https://github.com/CRIPAC-DIG/GraphAIR}.}. The key idea behind our approach is to \textit{explicitly model the neighborhood interaction in addition to neighborhood aggregation, which can better capture the complex and non-linear node features}.
WAs illustrated in Figure \ref{fig:node-interaction}, GraphAIR consists of two parts, i.e. \emph{aggregation} and \emph{interaction}. The aggregation module constructs node representations by combining features from neighborhoods; the interaction module explicitly models neighborhood interactions through multiplication.

\begin{figure}
	\centering
		\includegraphics[width=0.8\linewidth]{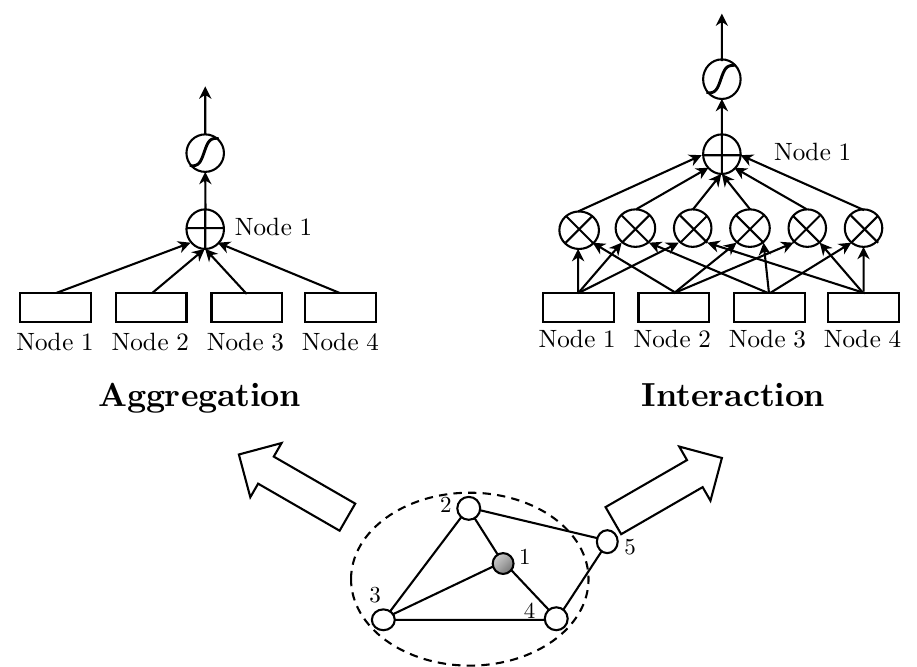}
	\caption{
	A graphical illustration of the proposed GraphAIR model. The aggregation module sums up the neighborhood features; the interaction module models the pair-wise feature interaction among the neighborhoods.}
	\label{fig:node-interaction}
\end{figure}

Nevertheless, several challenges exist in modeling the neighborhood interaction. Firstly, different nodes may have various numbers of adjacent neighbors, leading to different numbers of interaction pairs among neighbors. Thereby, defining a universal neighborhood interaction operator which is able to handle arbitrary numbers of interaction pairs is challenging. Secondly, it is preferable to propose a general plug-and-play interaction module instead of designing model-specific neighborhood interaction strategies for different GCN-based models.

To tackle the aforementioned challenges, we derive that the neighborhood interaction can be easily obtained through the multiplication of node embeddings.
As a result, both of the neighborhood aggregation module and the neighborhood interaction module can be implemented by most existing graph convolutional layers.

In a nutshell, the main contributions of this paper are three-fold.
Firstly, to best of our knowledge, it is the first work to analyze existing GCN-based methods from the perspective of capturing non-linearity of graph data. We then propose to explicitly model neighborhood interaction for capturing non-linearity of graph-structured data.
Secondly, the proposed GraphAIR can easily integrate off-the-shelf graph convolutional models, which shows favorable generality. Thirdly, extensive experiments conducted on benchmark tasks of node classification and link prediction show that GraphAIR achieves the state-of-the-art performance.

\section{Related Work}

Graph representation learning, which aims to learn low-dimensional representations that preserve useful topology and attributive information, plays an important role in many tasks, such as node classification \cite{Kipf:2016tc,Hu:2019vq,Zhu:2020wh,Zhu:2020vf}, graph classification \cite{Zhang:2020ed,Bai:2020wp}, clustering \cite{Yin:2018ux}, time-series analysis \cite{Bai:2020dn}, and recommendation \cite{Wu:2019ke,Yu:2020dn,Zhang:2020er,Cui:2019dr}.
There have been a lot of attempts in recent literature to employ neural networks for graph representation learning. Among them, graph convolutional neural networks (GCNs) receive a lot of research interests. GCN-based models generally follow the neighborhood aggregation scheme. To be specific, the model passes the input signals from neighborhoods through filters to aggregate information. Many approaches design different strategies to aggregate information from nodes' neighborhood. According to different strategies, these models can be roughly grouped into two categories, i.e. spectral-based approaches and spatial-based approaches.

One the one hand, spectral methods depend on the Laplacian eigenbasis to define parameterized filters. The first work \cite{Bruna:2014vg} introduce convolutional operations in the Fourier domain by computing the eigendecomposition of the graph Laplacian, which results in potentially heavy computational burden. Following its work, \cite{Defferrard:2016vo} propose to approximate filters using Chebyshev expansion of the graph Laplacian. Then, graph convolutional neural networks (GCNs) \cite{Kipf:2016tc} have been widely applied for graph representation learning. The core of GCNs is the neighborhood aggregation scheme which generates node embedding by combining information from neighborhoods. Since GCN only captures local information, DGCN \cite{Zhuang:2018ks} then proposes to construct an information matrix to encode global consistency.

On the other hand, the spatial approaches directly operate on spatially close neighbors. To enable parameter sharing of filters across neighbors of different sizes, \cite{Duvenaud:2015ww} first propose to learn weight matrices for different node degrees. MoNet \cite{Monti:2017ih} proposes a spatial-domain model to provide a unified convolutional network on graphs. To compute node representations in an inductive manner, GraphSAGE \cite{Hamilton:2017tp} samples fixed-size neighborhoods of nodes and performs aggregation over them. Similarly, \cite{Gao:2018dj} select a fixed number of neighbors and enable the use of conventional convolutional operations on Euclidean spaces. Recently, GAT \cite{Velickovic:2018we} introduces attention mechanisms to graph neural networks, which computes hidden representations by attending over neighbors with a self-attention strategy.
BASGCN \cite{Bai:2020fn} transforms the arbitrary-sized graphs into fixed-sized backtrackless aligned grid structures, and performs a novel backtrackless spatial graph convolutional operation on the grid structures to extract multi-scale local-level vertex features.

Recently, some methods are proposed to focus on linearity and non-linearity of graphs respectively. On the one hand, simplified graph convolutional networks (SGCs) \cite{Wu:2019vz} try to reduce the complexity and eliminate redundant computation of GCN by successively removing non-linear activation functions. SGC makes assumptions that non-linearity between GCN layers is not critical to the model performance and the majority of the benefit is brought by the neighborhood aggregation scheme. While being more computationally efficient, SGC achieves comparable empirical performance to vanilla GCN.

There are other methods arguing that modeling non-linear distributions of node features can bring improvements. For example, GraphSAGE-LSTM \cite{Hamilton:2017tp} employs the long-short-term memory (LSTM) module to learn the complex relationships between the nodes. Empirically, GraphSAGE-LSTM outperforms other aggregation functions such as GraphSAGE-mean and GraphSAGE-GCN. Graph isomorphic networks (GIN) \cite{Xu:2019ty} apply multilayer perceptrons (MLPs) in each graph convolutional layer, which is able to model complex non-linearity of graphs. Although theoretically it is well known that MLPs are universal approximators \cite{Hornik:1991te}, there is no formal theorem giving instructions on how to asymptotically approximate the desired function (\cite{Patterson:1996ug}, p.~182; \cite{Fausett:1994tq}, p.~328). Different from GraphSAGE-LSTM and GIN, to best of our knowledge, our work is the first to point out that most existing GCNs may not well capture non-linearity of graph data and we demonstrate the effectiveness of explicitly modeling non-linearity of graphs.

\section{Background and Preliminaries}

In this section, we firstly introduce the notations used throughout the paper and then summarize some of the most common GCN models. Last, we briefly introduce residual learning which we employ in our model.

\subsection{Notations}

Let \(\mathcal{G} = (\bm{A}, \bm{X})\) be an undirected graph with \(n\) nodes, where \(\bm{A} \in \mathbb{R}^{n \times n}\) is the adjacency matrix, \(\bm{X} \in \mathbb{R}^{n \times d}\) is the feature attribute matrix, and \(\bm{x}_i \in \mathbb{R}^{1 \times d}\) denotes the attribute of node \(i\). Please kindly note that in this paper we primarily focus on undirected graphs, but our proposed method can be easily generalized to work with weighted or directed graphs. The mathematical notations used throughout this paper are summarized in Table \ref{tab:notations}.

\begin{table*}
	\centering
	\caption{Notations used throughout this paper.}
	\begin{tabular}{cl}
	\toprule
	Notation & Description \\
	\midrule
	$\bm{A}, \widetilde{\bm{A}}$ & adjacency matrix, adjacency matrix with self-loops \\
	$n$ & the number of nodes \\
	$d$ & the dimension of the input feature \\
	$\mathcal{N}_i$ & the set of first-order neighbors of node $i$ including itself \\
	$\bm{n}_i$ & transformed embeddings before aggregation \\
	$\bm{W}^{(k)}$ & weight matrix in the $k$-th graph convolutional layer \\
	$e_{ij}$ & importance weight of node $j$'s feature to node $i$ \\
	$\beta_{ij}$ & interaction weight between node $i$ and node $j$ \\
	$\bm{h}_i$ & embedding of node $i$ resulting from graph convolution \\
	$\bm{h}_i^\text{agg}$ & embedding of node $i$ resulting from neighborhood aggregation \\
	$\bm{h}_i^\text{ir}$ & embedding of node $i$ resulting from neighborhood interaction \\
	$\bm{h}_i^\text{air}$ & embedding of node $i$ combining neighborhood aggregation and interaction \\
	$\bm{z}_i$ & embedding of node $i$ resulting from the output layer \\
	\bottomrule
	\end{tabular}
	\label{tab:notations}
\end{table*}

\subsection{Aggregators in Graph Convolutional Models}
\label{sec:aggregators}

As mentioned above, existing GCNs mainly differ in the neighborhood aggregation functions. The representative graph convolutional model such as GCN \cite{Kipf:2016tc} and GAT \cite{Velickovic:2018we} can be formulated as:
\begin{align}
	\bm{n}_i^{(k)} &= \bm{h}_i^{(k)} \bm{W}^{(k)}, \label{eq:weighted-representation}\\
	\bm{h}_i^{(k + 1)} &= \sigma\left( \sum_{j\in\mathcal{N}_i} e_{ij} \bm{n}_j^{(k)} \right), \label{eq:summarization-aggregator}
\end{align}
where \(\bm{h}_i^{(k)} \in \mathbb{R}^{d_k}\) is the embedding of the \(i\)\textsuperscript{th} node resulting from the \(k\)\textsuperscript{th} graph convolutional layer, \(\bm{W}^{(k)} \in \mathbb{R}^{d_k \times d_{k + 1}}\) is a learnable weight matrix, \(e_{ij}\) is a scalar which indicates the importance of node \(j\)'s features to node \(i\), and \(\sum_{j} e_{ij} = 1\). \(\sigma(\cdot)\) is the activation function, e.g., \(\operatorname{ReLU}(\cdot) = \max(0, \cdot)\) and \(\mathcal{N}_i\) is the set containing the first-order neighbors of node \(i\) as well as node \(i\) itself. To obtain the node embedding, a linear transformation is first conducted to project features to a new feature subspace. Then, the node embedding can be updated by weighted summation over the projected features of its neighbors, followed by a non-linear activation function.

Different models adopt different strategies to design the aggregators. For GCN, it uses a predefined weight matrix \(\hat{\bm{A}} = \tilde{\bm{D}}^{-\frac{1}{2}}\tilde{\bm{A}}\tilde{\bm{D}}^{-\frac{1}{2}}\) for summarization, where \(\tilde{\bm{A}} = \bm{A} + \bm{I}\) is the adjacency matrix with self-loops and \(\tilde{D}_{ii} = \sum_j \tilde{A}_{ij}\). Here, entry \(a_{ij}\) of \(\hat{\bm{A}}\) is a predefined weight factor for weighted summarization over neighborhoods, i.e. \(e_{ij}\) in Eq.(\ref{eq:summarization-aggregator}). Unlike GCN, GAT makes use of the attention mechanism to explicitly learn \(e_{ij}\) as follows:
\begin{equation}
	\begin{aligned}
		\alpha_{ij} &= g(\bm{n}_i, \bm{n}_j), \\
		e_{ij} &= \operatorname{softmax}(\alpha_{ij}) = \frac{\exp{(\alpha_{ij})}}{\sum_{k \in \mathcal{N}_i} \exp{(\alpha_{ik})}},
	\end{aligned}
\end{equation}
where \(g: \mathbb{R}^{d} \times \mathbb{R}^{d} \rightarrow \mathbb{R}\) is a self-attention function, which can be simply implemented as a feed-forward neural network.

\begin{table*}
	\centering
	\caption{Activation functions used in representation graph convolutional networks.}
	\label{tab:activation-functions}
	\begin{tabular}{ccc}
		\toprule
		Model & Propagation rule & Activation function \\
		\midrule
		GCN \cite{Kipf:2016tc} & $\bm{h}_i' = \sigma(\bm{W}\frac{1}{\operatorname{deg}(i)} \sum_{j\in\mathcal{N}_i} \bm{h}_j)$ & ReLU \\
		GraphSAGE-GCN \cite{Hamilton:2017tp} & $\bm{h}_i' = \sigma(\bm{W} \operatorname{avg}_{j \in \mathcal{N}_\text{sample}\cap\{i\}}(\bm{h}_j))$ & ReLU \\
		GAT \cite{Velickovic:2018we} & $\bm{h}_i' = \sigma(\sum_{j \in \mathcal{N}_i}e_{ij}\bm{h}_j)$ & ELU \\
		SGC \cite{Wu:2019vz} & $\bm{h}_i' = \bm{W}\frac{1}{\operatorname{deg}(i)} \sum_{j\in\mathcal{N}_i} \bm{h}_j$ & N.A. \\
		\bottomrule
	\end{tabular}
\end{table*}

\textbf{The implicit and insufficient neighborhood interaction involved in existing GCNs.}
It is seen from Eq. (\ref{eq:summarization-aggregator}) that without the activation function, the node representation would depend linearly on the neighborhood features. Then, although mainstream models adopt non-linear activation functions, which is able to introduce the neighborhood interaction implicitly as a side effect, they still face challenges in learning the neighborhood interaction sufficiently. We take the sigmoid function \(s(t) = \frac{1}{1+e^{-t}}\) as an example and approximate it with Taylor polynomials. Note that mainstream GCN-based models (as summarized in Table \ref{tab:activation-functions}) use piecewise non-saturating activation functions, such as \(\operatorname{ReLU}\) and \(\operatorname{LeakyReLU}(x)=\max(0.01x, x)\). These functions suppress negative values yet are still linear for positive values. Here we analyze the sigmoid function as it brings more non-linearity. Since the elements in the node embeddings are small\footnote{Most existing graph convolutional models, including GCN, GraphSAGE, and GAT normalize the input and initialize the weights using Glorot initialization \cite{Glorot:2010uc}.}, the high-order interacting terms among the neighborhoods are small as well. Then, we just analyze the coefficients of high-order interacting terms, which is claimed in the following proposition.
\begin{proposition}
	When applying the sigmoid function \(s(t)\) on the result of the linear combination as formulated in Eq. (\ref{eq:summarization-aggregator}), the equivalent coefficient of high-order interacting terms of the neighborhood embeddings is at most \(\frac{1}{48}\).
	\label{prop:implicit-interaction}
\end{proposition}

\begin{proof}
	The sigmoid function \(s(t)\) can be approximated as Taylor polynomials at \(t_0 = 0\):
	\begin{equation}
		s(t) \approx \sum_{p=0}^P \frac{s^{(p)}(0)}{p!} t^{p} = \frac{1}{2} + \frac{1}{4}t - \frac{1}{48}t^3 + \dots + \frac{s^{(P)}(0)}{P!} t^P,
	\end{equation}
	where \(P\) is the degree of the polynomial. The approximation error can be bounded using the Lagrange form of the remainder:
	\begin{equation}
		\begin{aligned}
			|R_p(t)| \leq & \frac{|t|^{P + 1}}{(P + 1)!}M_p, \quad \\
			& \text{where } |s^{P+1}(\theta)| \leq M_p, \quad \theta \in (-t, t).
		\end{aligned}
	\end{equation}
	Since the coefficient of the quadratic term is zero, we set \(P = 2\) and analyze the contribution of high-order interacting terms. Then, replacing \(t\) with \(\sum_{j \in \mathcal{N}_i} e_{ij}\bm{n}_j^{(k)}\), Eq. (\ref{eq:summarization-aggregator}) can be written as follows:
	\begin{equation}
		\begin{aligned}
		& \bm{h}_i^{(k + 1)} = \frac{1}{2} + \frac{1}{4} \left( \sum_{j \in \mathcal{N}_i} e_{ij}\bm{n}_i^{(k)} \right) \\
		& + M\left(\sum_{j \in \mathcal{N}_i}\sum_{k \in \mathcal{N}_i}\sum_{l \in \mathcal{N}_i} e_{ij}e_{ik}e_{il}\bm{n}_j^{(k)} \cdot \bm{n}_k^{(k)} \cdot \bm{n}_l^{(k)} \right),
		\end{aligned}
	\end{equation}
	where \(M\) is the bound of the reminder. To analyze its maximum value, we first get the third derivative of the sigmoid function:
	\begin{equation}
		\begin{aligned}
			s^{(3)}(\theta) & = \frac{\mathrm{d}^{3}}{\mathrm{d}\theta^{3}}\left(\frac{1}{1+e^{-\theta}}\right) \\
			& =\frac{e^{-\theta}}{\left(1+e^{-\theta}\right)^{2}}-\frac{6e^{-2\theta}}{\left(1+e^{-\theta}\right)^{3}}+\frac{6e^{-3\theta}}{\left(1+e^{-\theta}\right)^{4}}.
		\end{aligned}
	\end{equation}
	Then, making \(s^{(4)}(\theta) = 0\), we can calculate its roots:
	\begin{equation}
		\begin{aligned}
			\theta_1 & = 0,\\
			\theta_2 & = \log\left(5 + 2\sqrt{6}\right), \\
			\theta_3 & = \log\left(5 - 2\sqrt{6}\right).
		\end{aligned}
	\end{equation}
	Therefore, the corresponding extreme values of \(s^{(3)}(\theta)\) are \(-\frac{1}{8}\), \(\frac{1}{24}\), and \(\frac{1}{24}\). It is obvious to see the maximum absolute value of \(s^{(3)}(\theta)\) is \(\frac{1}{8}\). Therefore,
	\begin{equation}
		M = \frac{|s^{(3)}(\theta)|}{3!} \leq \frac{1}{48},
	\end{equation}
	which concludes the proof.
\end{proof}

\begin{remark}
Proposition \ref{prop:implicit-interaction} states that the effect of non-linear activation functions in GCNs is to introduce the interaction terms of neighborhood features.
The coefficients of the neighborhood interacting terms in current GCN-based models are relatively small, leading to a negligible contribution to node representations.
Considering other kinds of non-linear activation functions, such as ReLU-like piecewise non-saturating activation function, they exhibit less non-linearity than then sigmoid function; therefore, the coefficients of the neighborhood interacting terms in current GCN-based models will be smaller than $\frac{1}{48}$.
As existing GCNs are usually shallow with only one or two non-linear layers to avoid oversmoothing and overfitting \cite{Li:2018wc}, non-linearity of graph data cannot be learned sufficiently.
\end{remark}


Moreover, we also conduct empirical experiments to compare the performance of existing representative GCN-based models by using other activation functions, including sigmoid and tanh. As shown in Table \ref{fig:GCN-activation}, even with highly non-linear activation functions such as tanh and sigmoid, existing GCN-based methods perform even worse than their original models.
We think the performance degradation is due to these saturating non-linear activation functions suffer from vanishing gradients and are much slower than non-saturating  activation function \cite{Krizhevsky:2012wl,Xu:2015uj}. The results in Table \ref{fig:GCN-activation} demonstrates the inefficiency of leveraging activation functions to capture non-linearity of graph-structured data.

\begin{table*}
	\centering
	\caption{Node classification accuracy of representation GCN-based models with different activation functions.}
	\label{fig:GCN-activation}
	\begin{tabular}{ccccc}
		\toprule
		Model (Activation function) & Cora & Citeseer & Pubmed & NELL \\
		\midrule
		GCN (ReLU) & 81.5\% & 70.3\% & 79.0\% & 66.0\% \\
		GCN (sigmoid) & 79.4\% & 51.8\% & 77.3\% & 63.2\% \\
		GCN (tanh) & 80.6\% & 71.2\% & 79.4\% & 66.4\% \\
		GAT (ELU) & 83.0\% & 72.5\% & 79.0\% & -- \\
		GAT (sigmoid) & 36.0\% & 18.5\% & 54.4\% & -- \\
		GAT (tanh) & 82.6\% & 71.3\% & 77.5\% & -- \\
		\bottomrule
	\end{tabular}
\end{table*}

\subsection{Residual Learning}

In this paper, we employ residual learning to combine the neighborhood aggregation and interaction. Residual learning \cite{He:2016ib} is a widely-used building block for deep learning. Suppose \(h(\bm{x})\) is the true and desired mapping and \(\bm{x}\) is the suboptimal representation which serves as the input feature to the residual module. Residual learning can be formulated as:
\begin{equation}
	h(\bm{x}) = f(\bm{x}) + \bm{x},
\end{equation}
where \(f(\cdot)\) is a residual function. Practically, we can apply a few non-linear layers to obtain the suboptimal representation \(\bm{x}\) and some other non-linear layers to implement the residual function \(f\). The essence of residual learning lies in the skip connection, through which the earlier representations are able to flow to later layers. The skip connection enables more direct reuse of the suboptimal representation and improves the information flow during forward and backward propagation \cite{He:2016ib}, which makes the network easier to be optimized. Many approaches \cite{He:2016ib,Li:2018ta} have shown that residual learning helps break away from the local optimum and improving the performance.

\section{The Proposed Method: GraphAIR}

In this section, we firstly formulate the model of neighborhood interaction and then describe how the parameters of GraphAIR model can be learned. Finally, we summarize the overall model architecture and analyze the computational complexity.

\subsection{Modeling the Neighborhood Interaction with Residual Functions}

As discussed in Section \ref{sec:aggregators}, the node representation resulting from the neighborhood aggregation scheme is less likely to well capture complicated non-linearity of graphs because they learn the neighborhood interaction implicitly and inefficiently. In this section, we describe the embedding generation algorithm of GraphAIR, which aims to incorporate the neighborhood interaction into node representations.
To begin with, a natural idea to model the quadratic terms of neighborhood interaction is formulated as:
\begin{equation}
	\bm{h}_i^\text{ir} = \sum_{j\in\mathcal{N}_i} \sum_{k\in\mathcal{N}_i} \beta_{jk} \bm{n}_j \odot \bm{n}_k,
\end{equation}
where \(\bm{h}_i^\text{ir}\) is the neighborhood \underline{i}nte\underline{r}action representation of node \(i\), \(\beta_{jk}\) denotes the coefficient of the quadratic term, and \(\odot\) is the element-wise multiplication operator.
\begin{figure}
	\centering
	\subfloat[Self-interaction module can represent neighborhood interaction.]{
		\includegraphics[scale=0.6]{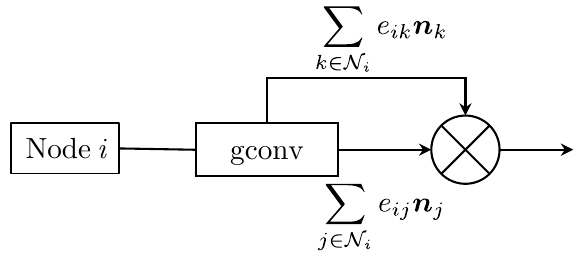}
		\label{fig:self-interaction}
	}\\
	\subfloat[The neighborhood interaction module applies an auxiliary layer for better representation.]{
		\includegraphics[scale=0.6]{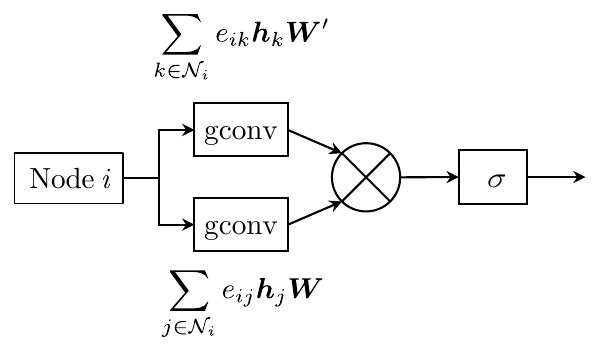}
		\label{fig:neighborhood-interaction}
	}\\
	\subfloat[Overview of node aggregation and interaction in GraphAIR. ``gconv'' block indicates the general graph convolutional layer, which can be instantiated as GCN \cite{Kipf:2016tc}, GAT \cite{Velickovic:2018we}, etc.]{
		\includegraphics[scale=0.6]{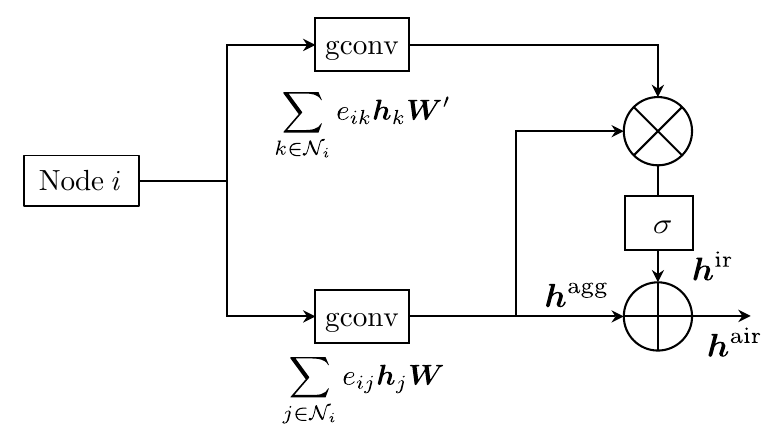}
		\label{fig:node-aggregation-interaction}
	}
	\caption{Detailed illustration of neighborhood aggregation and interaction in GraphAIR.}
	\label{fig:details}
\end{figure}
However, it is infeasible to learn \(\beta_{jk}\) in our case. For each node \(i\), there are \(O(|\mathcal{N}_i|^2)\) coefficients to estimate, which exposes the risk of overfitting. To alleviate this problem, we simply assign \(\beta_{jk}\) as the product of importance weights \(e_{ij}\) and \(e_{ik}\). The simplification is reasonable with the following aspects. For node \(i\), if \(e_{ij}\) and \(e_{ik}\) are large, then the neighbor nodes \(j\) and \(k\) should be considered as important factors for the representation of node \(i\). Compared to other interacting terms, the interaction between node \(j\) and \(k\) are likely to provide more relevant information about node \(i\). Consequently, \(\beta_{jk}\) should be large.
In contrast, if \(e_{ij}\) and \(e_{ik}\) are small, neighbor nodes \(j\) and \(k\) may have a slight impact on node \(i\). Thus, the interacting coefficient should be small as well. Formally, we arrive at:
\begin{equation}
	\begin{aligned}
		\bm{h}_i^\text{ir} & = \left(\sum_{j \in \mathcal{N}_i} e_{ij} \bm{n}_j\right) \odot \left(\sum_{k \in \mathcal{N}_i} e_{ik} \bm{n}_k\right) \\
		& = \left(\sum_{j \in \mathcal{N}_i} e_{ij} \bm{h}_j \bm{W}\right) \odot \left(\sum_{k \in \mathcal{N}_i} e_{ik} \bm{h}_k \bm{W}\right) \\
		& = \bm{h}_i^\text{agg} \odot \bm{h}_i^\text{agg},
	\end{aligned}
	\label{eq:preliminary-interaction}
\end{equation}
where \(\bm{h}_i^\text{agg} = \sum_{j \in \mathcal{N}_i} e_{ij} \bm{h}_j \bm{W}\) denotes the representation resulting from neighborhood aggregation. Figure \ref{fig:self-interaction} illustrates this process. For brevity, we name this operation as self-InteRaction (self-IR)

In order to introduce more non-linearity to our model, we apply non-linear activation function on the two representations resulting from neighborhood aggregation and neighborhood interaction respectively.
Since combining different non-linear layers by skip connection for capturing various information has been widely used in recent work \cite{He:2016ib,Xu:2018vn,Lee:2019ww}, we combine these two representations by using a skip connection as follows.
\begin{equation}
	\bm{h}_i^\text{air} = \sigma(\bm{h}_i^\text{agg}) + \sigma(\bm{h}_i^\text{ir}).
	\label{eq:preliminary-residual}
\end{equation}
It should be noted that concatenation is also widely used and we found that skip connections by adding performs better for combining neighborhood aggregation and interaction. Please refer to Section \ref{sec:combining-aggregation-interaction} for additional experiments.

However, although we adopt a skip connection here, we argue that we still cannot benefit from residual learning, where both of the suboptimal representation and the residual function are implemented by \emph{different non-linear layers}. As formulated in Eqs. (\ref{eq:preliminary-interaction},\ref{eq:preliminary-residual}), the two representations resulting from neighborhood aggregation and interaction are based on the \emph{same} weight matrix \(\bm{W}\), which means the variations of the two representations during the back-propagation process are highly correlated. According to \cite{Bengio:2013bu}, it is important to disentangle the factors of variation to the representations as only a few factors tend to change at a time. Therefore, 
to make use of residual learning which can ease the optimization, we introduce another weight matrix \(\bm{W}^\prime \in \mathbb{R}^{d_k \times d_{k + 1}}\) to disentangle learning the neighborhood interaction from neighborhood aggregation. Formally, instead of Eq. (\ref{eq:preliminary-interaction}), we use the following equation to learn the neighborhood interaction in our model: 
\begin{equation}
	\begin{aligned}
		\bm{h}_i^\text{ir} & = \left(\sum_{j \in \mathcal{N}_i} e_{ij} \bm{h}_j \bm{W}\right) \odot \left(\sum_{k \in \mathcal{N}_i} e_{ik} \bm{h}_k \bm{W}^\prime\right) \\
		& = \bm{h}_i^\text{agg} \odot \bm{\bar{h}}_i^\text{agg},
		\label{eq:node-interaction}
	\end{aligned}
\end{equation}
where the first term \(\bm{h}_i^\text{agg} = \sum_{j \in \mathcal{N}_i} e_{ij} \bm{h}_j \bm{W}\) denotes the representation resulting from neighborhood aggregation and the second term \(\bm{\bar{h}}_i^\text{agg} = \sum_{k \in \mathcal{N}_i} e_{ik} \bm{n}_k \bm{W}^\prime\) provides the other half node representation for multiplication in the interaction process. \(\bm{h}_i^\text{agg}\) is the input representation to the residual module and \(\bm{W}^\prime\) is the learnable weight of the residual function. Note that both terms \(\bm{h}_i^\text{agg}\) and \(\bm{\bar{h}}_i^\text{agg}\) can be implemented by existing graph convolutional layers.
We illustrate Eq. (\ref{eq:node-interaction}) in Figure \ref{fig:neighborhood-interaction}. We also conduct ablation study in Section 5.4 to prove the effectiveness of Eq. (\ref{eq:node-interaction}).

To sum up, the neighborhood aggregation and interaction module shown in Figure \ref{fig:node-interaction} can be combined by skip connections, as shown in Figure \ref{fig:node-aggregation-interaction}.
Thus, the proposed GraphAIR framework is compatible with most existing GCN-based models such as GCN \cite{Kipf:2016tc} and GAT \cite{Velickovic:2018we} and it provides a plug-and-play module for the neighborhood interaction.

%
%
%

\subsection{Learning the Parameters of GraphAIR}

In this section, we introduce how to learn the parameters under the GraphAIR framework. As we aim to propose a general approach for graph representation learning, we can apply different kinds of graph-based loss function, such as the proximity ranking loss in link prediction tasks and the cross-entropy loss in node classification tasks. Without loss of generality, we take the task of node classification as an example.

To compute the probability that each node belongs to a certain class, existing GCN-based models usually employ one additional graph convolutional layer with a softmax classifier for prediction. Then, the output representation \(\bm{z}_i\) is formulated as:
\begin{equation}
	\bm{z}_i = g\left(\bm{h}_i^{(k)}\right) = \operatorname{softmax} \left( \sum_{j\in\mathcal{N}_{i}}e_{ij} \bm{h}_j^{(k)} \bm{W}^{(k + 1)} \right),
	\label{eq:output}
\end{equation}
where \(g(\cdot)\) is the prediction function, \(\bm{W}^{(k + 1)} \in \mathbb{R}^{d_k \times |\mathcal{Y}|}\), and \(|\mathcal{Y}|\) is the number of classes. Then, the loss of node classification can be calculated as
\(\mathcal{L} = \frac{1}{n}\sum_{i = 1}^n \mathcal{L}_\text{clf}(\bm{z}_i, \bm{y}_i)\)
where \(\bm{y}_i\) is the true label for node \(i\) and \(\mathcal{L}_\text{clf}\) is the cross-entropy loss.

To obtain more accurate node embeddings \(\bm{h}_i^\text{agg}\) and \(\bm{\bar{h}}_i^\text{agg}\), we apply two auxiliary classifiers on \(\bm{h}_i^\text{agg}\) and \(\bm{\bar{h}}_i^\text{agg}\). Subsequently, the resulting representation \(\bm{h}_i^\text{ir}\) for the neighborhood interaction will be more precise as well. Then, as formulated in Eq. (\ref{eq:output}), we apply one additional graph convolutional layer on each of \(\bm{h}_i^\text{air}\), \(\bm{h}_i^\text{agg}\), and \(\bm{\bar{h}}_i^\text{agg}\) to attain \(\bm{z}_i^\text{air}\), \(\bm{z}_i^\text{agg}\), and \(\bm{\bar{z}}_i^\text{agg}\). Eventually, the overall objective function is the weighted sum of the three losses:
\begin{equation}
	\begin{aligned}
		\mathcal{L}_\text{total} &= \frac{1}{n}\sum_{i = 1}^n \left[ \lambda_1 \mathcal{L}_\text{clf}(\bm{z}_i^\text{air}, \bm{y}_i) + \lambda_2 \mathcal{L}_\text{clf}(\bm{z}_i^\text{agg}, \bm{y}_i) + \lambda_3 \mathcal{L}_\text{clf}(\bm{\bar{z}}_i^\text{agg}, \bm{y}_i) \right] \\
		&= \lambda_1 \mathcal{L}^\text{air} + \lambda_2 \mathcal{L}^\text{agg} + \lambda_3 \mathcal{\bar{L}}^\text{agg}, \\
	\end{aligned}
\end{equation}
where \(\lambda_1\), \(\lambda_2\), and \(\lambda_3\) are hyperparameters controlling weights of the three loss functions. For training, we minimize the total loss \(\mathcal{L}_\text{total}\), while for inference, we only use \(\bm{z}_i^\text{air}\), since \(\bm{z}_i^\text{agg}\) and \(\bm{\bar{z}}_i^\text{air}\) are to ensure \(\bm{h}_i^\text{air}\) is accurate enough.

\subsection{Model Architecture and Complexity Analysis}

We suppose there are \((K + 1)\) layers in the underlying graph convolutional model, where the last layer is employed for node classification. For GraphAIR, we employ two separate and symmetric branches, each of which consists of \(K\) graph convolutional layers to obtain \(\bm{h}_i^\text{agg}\) and \(\bm{\bar{h}}_i^\text{agg}\). Then, considering \(\bm{h}_i^\text{agg}\) and \(\bm{\bar{h}}_i^\text{agg}\) have aggregated enough information from neighborhoods, here we conduct the neighborhood interaction only once by multiplying \(\bm{h}_i^\text{agg}\) and \(\bm{\bar{h}}_i^\text{agg}\) for the sake of efficiency. Additionally, we employ three graph convolutional layers followed by softmax activation functions on \(\bm{h}_i^\text{air}\), \(\bm{h}_i^\text{agg}\), and \(\bm{\bar{h}}_i^\text{agg}\). In summary, there will be \((2K + 3)\) layers in GraphAIR.

\begin{figure*}
	\centering
	\includegraphics[width=0.75\linewidth]{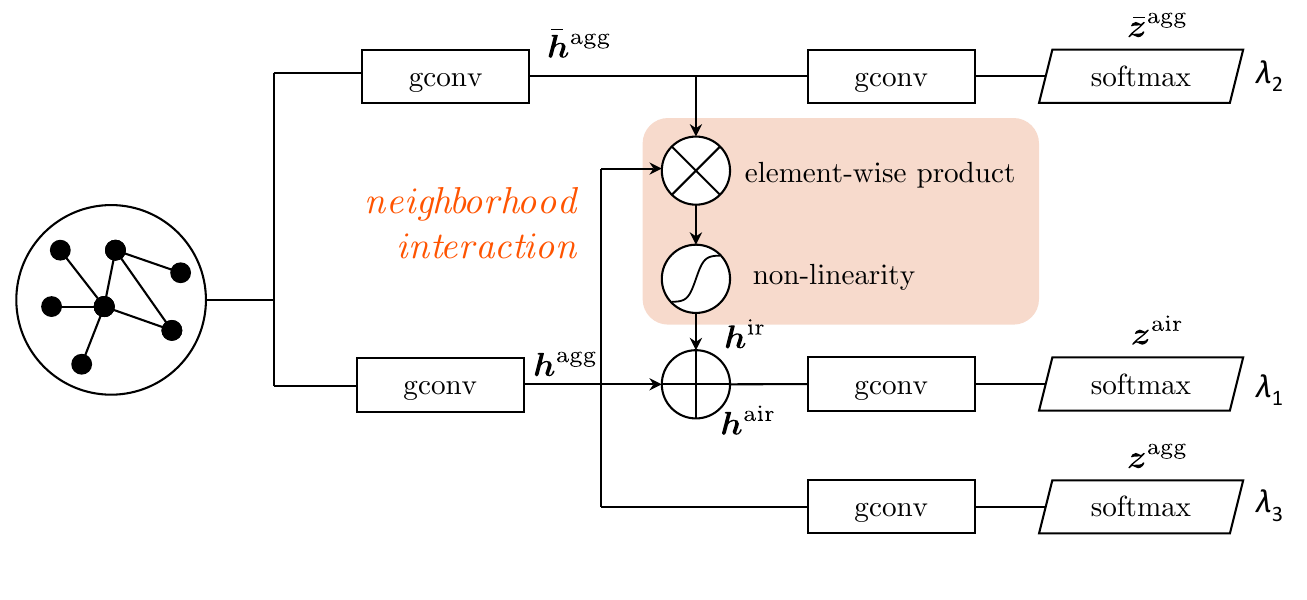}
	\caption{The proposed GraphAIR framework. The ``gconv'' block indicates the general graph convolutional layer, which can be instantiated as GCN layers, GAT layers, etc.}
	\label{fig:framework}
\end{figure*}

The proposed GraphAIR model with \(K = 1\) is illustrated in Figure \ref{fig:framework}.
Each layer in GraphAIR has the same space and time complexity as the underlying model and the additional computation cost of GraphAIR is mainly introduced by the multiplication process for the neighborhood interaction. For the neighborhood interaction in Eq. (\ref{eq:node-interaction}), the cost is \(O(nd)\) where \(d\) is the embedding dimension. For each layer of the existing graph convolutional model such as GCN and GAT, it takes \(O(n^2d)\) time to proceed Eq. (\ref{eq:summarization-aggregator}). Therefore, the additional computation cost of neighborhood interaction is insignificant. That is to say, our proposed approach is as asymptotically efficient as the underlying graph convolutional model.

\section{Evaluation}
We extensively evaluate our proposed GraphAIR model on the node classification task and link prediction using five public datasets. Besides, we also conduct ablation studies on the neighborhood interaction module. For readers of interest, we include comparison of training time and all details of the experimental configurations in the supplementary material.


\subsection{Datasets}

We use five widely-used datasets to evaluate model performance on both transductive learning and inductive learning scenarios. Specifically, three citation networks (Cora, Citeseer, Pubmed) are used for tranductive node classification and link prediction, one knowledge graph (NELL) is used for transductive node classification, and one multi-graph molecular network (PPI) is for inductive node classification. We exactly follow the setup in \cite{Yang:2016ts,Kipf:2016tc,Kipf:2016ul,Velickovic:2018we}. The statistics of datasets used throughout the experiments are summarized in Table \ref{tab:datasets}.

\textbf{Citation networks.}
We build undirected citation networks from three datasets, where documents and citations are treated as nodes and edges respectively. We treat the bag-of-words of each document as the feature vector. Our goal is to predict the class of each document. Only twenty labels per class are used for training.

\textbf{Knowledge graphs.}
The dataset collected from the knowledge base of Never Ending Language Learning (NELL) contains entities, relations, and text description. For every triplet \((e_1, r, e_2)\), where \(e_1\) and \(e_2\) are entities and \(r\) is the relationship between them, \(r\) will be assigned with two separate nodes \(r_1\) and \(r_2\). Then, we add two edges between \((e_1, r_1)\) and \((e_2, r_2)\). For the knowledge graph, we conduct the entity classification. Similarly, we use bag-of-words as feature vectors. Only one label per class is used for training.

\textbf{Molecular networks.}
We use the PPI (protein-protein interaction) network that consists of twenty-four (24) graphs corresponding to different human tissues. Each node contains fifty (50) features composed of positional gene sets, motif gene sets, and immunological signatures. We select twenty (20) graphs as the training set, two (2) for validation, and two (2) for testing.

\begin{table*}
	\centering
	\caption{Dataset statistics.}
	\begin{tabular}{cccccccccc}
		\toprule
		Dataset & Cora & Citeseer & Pubmed & NELL & PPI \\
		\midrule
		Task  & \multicolumn{4}{c}{Transductive} & Inductive \\
		Type  & \multicolumn{3}{c}{Citation network} & Knowledge graph & Molecular \\
		\# Vertices & 2,708 & 3,327 & 19,717 & 65,755 & 56,944 \\
		\# Edges & 5,429 & 4,732 & 44,338 & 266,144 & 818,716 \\
		\# Classes & 7     & 6     & 3     & 210   & 121 \\
		\# Features & 1,433 & 3,703 & 500   & 5,414 & 50 \\
		\# Training nodes & 140   & 120   & 60    & 210   & 44,906 \\
		\# Test nodes & 1,000 & 1,000 & 1,000 & 1,000 & 5,524 \\
		\# Validation nodes & 500   & 500   & 500   & 500  & 6,514 \\
		\bottomrule
	\end{tabular}
	\label{tab:datasets}
\end{table*}

\subsection{Experiments on Node Classification}

\subsubsection{Baseline Methods}
We comprehensively compare our method with various traditional random-walk-based algorithms and state-of-the-art GCN-based methods. We closely follow the experimental setting of previous work; the performance of those baselines is reported as in their original papers.\footnote{In experiments, we found that the results reported in \cite{Hamilton:2017tp} after ten epochs did not converge to the best values. For a fair comparison with other models, we reuse its official implementation and report the results of the baselines after 200 epochs.}

\textbf{Transductive node classification.}
In the transductive setting, the baselines include skip-gram-based network embedding method DeepWalk \cite{Perozzi:2014ib}, graph convolutional networks with higher-order Chebyshev filters (Planetoid) \cite{Yang:2016ts}, graph convolution with one-hop neighbors (GCN) \cite{Kipf:2016tc}, and graph attention networks (GAT) \cite{Velickovic:2018we}. In addition, we further compare the performance of the proposed model with the recently proposed simplified graph convolutional networks (SGCs) \cite{Wu:2019vz} which removes redundant non-linear activations. Also, we modify graph isomorphic networks (GINs) \cite{Xu:2019ty} which utilize non-linear MLPs as the aggregation function for the node classification task. Note that since GIN was originally proposed for graph classification, we apply two GIN convolutional layers and remove the graph-level readout function for the transductive node classification task.
Moreover, we include two recent methods Graph Markov Neural Networks (GMNN) \cite{Qu:2019wz} and Deeper Graph Convolutional Networks (DAGNN) \cite{Liu:2020dt}.

\textbf{Inductive node classification.}
For inductive node classification, we mainly compare GraphAIR with inductive graph convolutional networks (GraphSAGE) \cite{Hamilton:2017tp} and graph attention networks (GAT) \cite{Velickovic:2018we}. Note that GraphSAGE provides several variants of neighborhood aggregators: SAGE-mean concatenates the features of the neighborhoods and the central node, SAGE-GCN takes the average over neighborhood feature vectors, SAGE-LSTM combines neighborhood features by using a LSTM model, and SAGE-pool uses an element-wise max-pooling operator to aggregate the neighborhood information nonlinearly.


\subsubsection{Experimental Configurations}
We employ our GraphAIR framework on top of three representative models, including GCN, GraphSAGE, and GAT, which is denoted by AIR-GCN, AIR-SAGE, and AIR-GAT, respectively. Particularly, while GraphSAGE proposes several variants for neighborhood aggregation, among them only SAGE-GCN satisfies the coefficient normalization in Eq. (\ref{eq:node-interaction}). Therefore, we select SAGE-GCN as the base model for GraphAIR. For a fair comparison, we closely follow the same hyper-parameters setting as the underlying graph convolutional model, such as learning rate, dropout rate, weight decay factor, hidden dimensions, etc. Considering GIN is originally proposed for graph-level classification, the hidden dimensions are set to the same as GCN. In the experiment, we only tune the weights of three loss functions by grid search, where $\lambda_i \in [0.1, 0.2, \dots, 1.5], \forall i \in \{1,2,3\}$. For the transductive setting, we use the features of all data but only the labels of the training set are used for training. For the inductive setting, we train our model without the validation data and testing data. In addition, we report the average accuracy of 20 measurements.

\subsubsection{Results and Analysis}

\textbf{Transductive.}
We summarize the results of transductive node classification in Table \ref{tab:transductive-results}. Note that even we apply the sparse version implementation of GAT, it requires more than 64G memory on NELL dataset. Thus, the performance of GAT and AIR-GAT is not reported. From the tables, it is seen that GraphAIR achieves state-of-the-art performance over all datasets, which demonstrates the effectiveness of the proposed GraphAIR framework. SGC acquires comparable results to that of GCN, which corresponds to our conclusion in Proposition 1 that existing GCNs are not able to learn the nonlinearity of graph data sufficiently. For our proposed AIR-GCN, it outperforms its base model GCN by margins of 3.2\%, 2.6\%, 1.0\%, and 2.5\%. The same trends hold for AIR-GAT with its base model GAT as well. To sum up, the improvements demonstrate the effectiveness of modeling the non-linear distributions of nodes.

In addition, another important observation is that, both AIR-GAT and AIR-GCN outperform the complex non-linear opponents such as GIN. Although MLPs are able to asymptotically approximate any complicated and non-linear functions theoretically, they tend to converge to undesired local minima in practice \cite{Rojas:1996wf}.
The experimental results prove the rationality of \emph{explicitly} introducing neighborhood interaction.

\begin{table*}
	\centering
	\caption{Accuracy of transductive node classification with the best performance highlighted in bold.}
	\begin{tabular}{ccccc}
		\toprule
		Method & Cora & Citeseer & Pubmed & NELL \\
		\midrule
		DeepWalk & 67.2\% & 43.2\% & 65.3\% & 58.1\% \\
		Planetoid & 75.7\% & 64.7\% & 77.2\% & 61.9\% \\
		GIN+0 & 78.3\% & 62.9\% & 78.0\% & 65.5\% \\
		GIN+\(\epsilon\) & 76.6\% & 63.8\% & 75.5\% & 63.5\% \\
		GCN   & 81.5\% & 70.3\% & 79.0\% & 66.0\% \\
		SGC   & 81.0\% \(\pm\) 0.0\% & 71.9\% \(\pm\) 0.1\% & 78.9\% \(\pm\) 0.0\% & 65.4\% \(\pm\) 0.2\% \\
		GAT   & 83.0\% \(\pm\) 0.7\% & 72.5\% \(\pm\) 0.7\% & 79.0\% \(\pm\) 0.3\% & -- \\
		GMNN  & 83.7\% & 72.9\% & \textbf{81.8\%} & 68.0\% \\
		DAGNN & 84.4\% \(\pm\) 0.5\% & 73.3\% \(\pm\) 0.6\% & 80.5\% \(\pm\) 0.5\% & 67.5\% \(\pm\) 0.3 \% \\
		\midrule
		AIR-GCN & \textbf{84.7\% \(\pm\) 0.1\%} & \textbf{72.9\% \(\pm\) 0.1\%} & 80.0\% \(\pm\) 0.1\% & \textbf{68.5\% \(\pm\) 0.2\%} \\
		AIR-GAT & \textbf{84.5\% \(\pm\) 0.7\%} & \textbf{73.5\% \(\pm\) 0.6\%} & 80.0\% \(\pm\) 0.2\% & -- \\
		\bottomrule
	\end{tabular}
	\label{tab:transductive-results}
\end{table*}

\begin{table}
	\centering
	\caption{Accuracy of inductive node classification with the best performance highlighted in bold.}
	\begin{tabular}{cc}
		\toprule
		Method & PPI \\
		\midrule
		Random & 39.6\% \\
		SAGE-GCN & 55.6\% \\
		SAGE-mean & 64.5\% \\
		SAGE-LSTM & 66.8\% \\
		SAGE-pool & 73.8\% \\
		GAT & 97.3\% \(\pm\) 0.2\% \\
		\midrule
		AIR-SAGE-GCN & 57.8\% \(\pm\) 0.1\%\\
		AIR-GAT & \textbf{98.6\% \(\pm\) 0.2\%} \\
		\bottomrule
	\end{tabular}
	\label{tab:inductive-results}
\end{table}

\textbf{Inductive.}
The results of inductive learning are shown in Table \ref{tab:inductive-results}. AIR-SAGE-GCN outperforms its base model SAGE-GCN by 2.2\%. Besides, we can clearly observe that AIR-GAT achieves the best performance. It is worth noting that the previous state-of-the-art method has already reached pretty high performance and the proposed AIR-GAT still acquires the improvement of 1.3\% over the vanilla GAT. Besides, it is suggested that the proposed GraphAIR framework is also generalizable for multiple graphs.

\subsection{Experiments on Link Prediction}

In order to further verify our proposed framework is general for other graph representation learning tasks, we conduct experiments on link prediction additionally. We choose citation networks as benchmark datasets and compare against various state-of-the-art methods, including graph autoencoders (GAE) \cite{Kipf:2016ul} and variational graph autoencoders (VGAE) \cite{Kipf:2016ul}, as well as other baseline algorithms, including SC \cite{Tang:2011cx} and DeepWalk \cite{Perozzi:2014ib}. We employ our GraphAIR framework on the basis of GAE, which constructs the graph autoencoder with GCNs. The resulting model is denoted by AIR-GAE.

We report the performance in terms of area under the ROC curve (AUC) based on the performance of 20 runs. The mean performance and standard error are presented in Table \ref{tab:link-prediction}. It is shown from the table that the proposed AIR-GAE outperforms its vanilla opponents GAE and VGAE, which once again verifies the necessity to incorporate the neighborhood interaction to neighborhood aggregation. Please note that previous state-of-the-art methods have already obtained high enough performance on the Pubmed dataset and our method AIR-GAE pushes the boundary with absolute improvements of 2.8\%, achieving 99.2\% in terms of AUC. Also, it can be observed that the proposed method obtains much more obvious improvements, compared with the performance of node classification. We suspect that this is primarily because models for the link prediction task usually employ pairwise decoders for calculating the probability of the link between two nodes. For example, GAE and VGAE assume the probability that there exists an edge between two nodes is proportional to the dot product of the embeddings of these two nodes. Therefore, our approach, which explicitly models the neighborhood interaction through the multiplication of the embeddings of two nodes, is inherently related to the link prediction task and obtains more improvements.

\begin{table}
	\small
	\centering
	\caption{AUC of link prediction in citation networks with the best performance highlighted in bold.}
	\label{tab:link-prediction}
	\resizebox{\linewidth}{!}{
	\begin{tabular}{ccccccc}
	\toprule
	Method   & Cora  & Citeseer & Pubmed \\
	\midrule
	SC & 84.6\% \(\pm{}\) 0.01\% & 80.5\% \(\pm{}\) 0.01\% & 84.2\% \(\pm{}\) 0.02\% \\
	DeepWalk & 83.1\% \(\pm{}\) 0.01\% & 80.5\% \(\pm{}\) 0.02\% & 84.4\% \(\pm{}\) 0.00\% \\
	\midrule
	GAE & 91.0\% \(\pm{}\) 0.02\% & 89.5\% \(\pm{}\) 0.04\% & 96.4\% \(\pm{}\) 0.00\% \\
	VGAE & 91.4\% \(\pm{}\) 0.01\% & 90.8\% \(\pm{}\) 0.02\% & 94.4\% \(\pm{}\) 0.02\% \\
	\midrule
	AIR-GAE & \textbf{95.4\% \(\pm{}\) 0.01\%} & \textbf{95.0\% \(\pm{}\) 0.01\%}  & \textbf{99.2\% \(\pm{}\) 0.02\%}  \\
	\bottomrule
	\end{tabular}
	}
\end{table}

\subsection{Ablation Studies on the Neighborhood Interaction Module}

As we analyzed in Section 3.3, the number of parameters in GraphAIR is almost two times than that of the underlying graph convolutional model. In this section, we conduct ablation studies to answer the following questions:
\begin{itemize}
	\item \textbf{Q1}: How much improvement has the proposed neighborhood interaction module brought?
	\item \textbf{Q2}: Does the disentangled residual learning strategy bring sufficient improvements?
\end{itemize}

To answer Q1 and verify the effectiveness of GraphAIR is introduced by the proposed neighborhood interaction module rather than the larger number of parameters in the model, we remove the neighborhood interaction module of AIR-GCN. Then, the resulting model has exactly the same parameters as AIR-GCN. As there are almost double parameters than vanilla GCN in the resulting model, we denote the resulting model as DP-GCN (Double-Parameter GCN).

To answer Q2, we employ only one branch of graph convolutional networks consisting \((K+1)\) layers to produce the output representations. To obtain neighborhood interaction \(\bm{h}^\text{ir}\), we directly make use of the self-interaction strategy described in Eq. (9) instead of Eq. (11). The resulting model is termed as self-IR-GCN.

\begin{figure}
	\centering
	\includegraphics[width=0.8\linewidth]{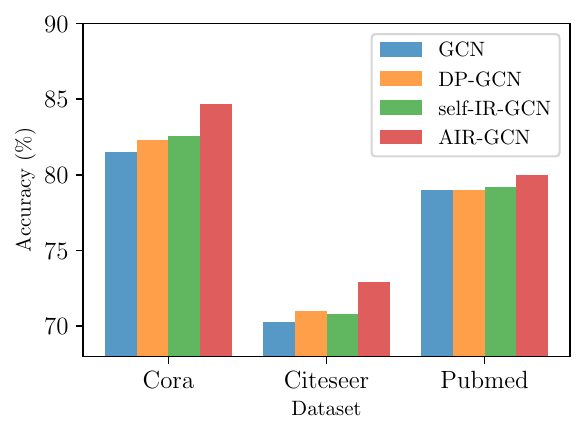}
	\caption{Accuracy of node classification in the ablation study.}
	\label{fig:ablation-study}
\end{figure}

For fair comparison, other experimental configurations are kept the same as AIR-GCN. The results of node classification are presented in Figure \ref{fig:ablation-study}. It is seen from the figure that the proposed AIR-GCN evidently achieves the best performance and outperforms DP-GCN and self-IR-GCN. For Q1, we can observe that DP-GCN only obtains slightly better accuracy on Cora and Citeseer and almost the same performance as the vanilla GCN on Pubmed. It can be verified that the neighborhood interaction module mainly contributes to the performance improvement of the proposed AIR-GCN model. For Q2, it is seen that the performance of self-IR-GCN only gets slightly improved on three datasets, which demonstrates the rationality of modeling neighborhood interaction. However, disengaging the neighborhood interaction from neighborhood aggregation can bring more improvements.

\begin{table*}
	\centering
	\caption{Accuracy of node classification with different ways of combining node embeddings resulting from neighborhood aggregation and interaction.}
	\label{tab:combine-neighborhood-interaction}
	\begin{tabular}{ccccc}
		\toprule
		Method & Cora & Citeseer & Pubmed & NELL \\
		\midrule
		AIR-GCN & \textbf{84.7\% \(\pm\) 0.1\%} & \textbf{72.9\% \(\pm\) 0.1\%} & 80.0\% \(\pm\) 0.1\% & \textbf{68.5\% \(\pm\) 0.2\%} \\
		AIR-GAT & \textbf{84.5\% \(\pm\) 0.7\%} & \textbf{73.5\% \(\pm\) 0.6\%} & 80.0\% \(\pm\) 0.2\% & -- \\
		\midrule
		AIR-GCN-concat & 83.1\% \(\pm\) 0.1\% & 70.6\% \(\pm\) 0.1\% & 79.2\% \(\pm\) 0.1\% & 66.8\% \(\pm\) 0.1\% \\
		AIR-GAT-concat & 83.3\% \(\pm\) 0.8\% & 71.2\% \(\pm\) 0.6\% & 79.4\% \(\pm\) 0.3\% & -- \\
		\bottomrule
	\end{tabular}
\end{table*}

\subsection{Discussions on Combining Neighborhood Aggregation and Interaction}
\label{sec:combining-aggregation-interaction}

In this section, we also examine different ways of combining node embeddings resulting from neighborhood aggregation and interaction. Apart from residual connections \cite{He:2016ib}, we also conduct experiments that combine node embeddings from neighborhood aggregation and interaction using concatenation, where the models are denoted as AIR-GCN-concat and AIR-GAT-concat respectively. The results are summarized in Table \ref{tab:combine-neighborhood-interaction}. It is seen that although using concatenation introduces no information loss, it may cause performance loss due to overfitting.

\subsection{Parameter Sensitivity Analysis}

\begin{figure}
	\centering
	\includegraphics[width=\linewidth]{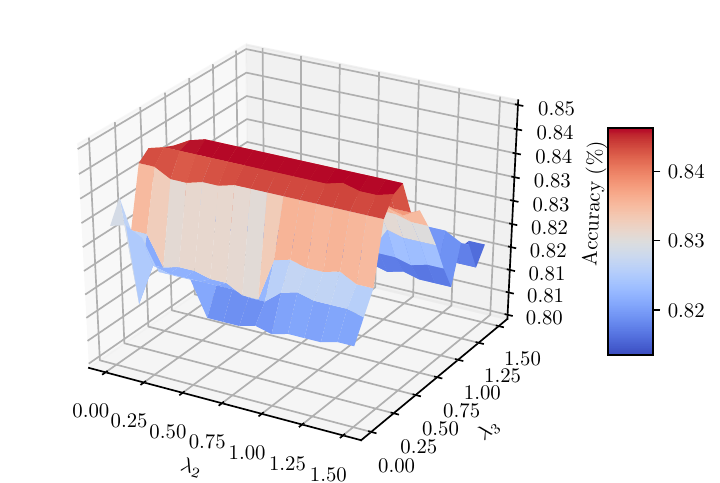}
	\caption{Accuracy of node classification on Cora with varying hyperparameters.}
	\label{fig:sensitivity}
\end{figure}

In this section, we attempt to investigate how the hyperparameters of loss functions impact the model performance. To this end, we conduct experiments on the Cora dataset with different combinations of parameters \((\lambda_2, \lambda_3)\), where \(\lambda_2, \lambda_3 \in \{0, 0.1, 0.2, \dots ,1.5\}\). We report the node classification performance in terms of accuracy and the corresponding results are shown in Figure \ref{fig:sensitivity}.
For other hyperparameters, we set \(\lambda_1 = 1.1\) according to the performance on most datasets, which is also consistent with our previous experiments reported in Section 5.2 and Section 5.3. It can be seen from Figure \ref{fig:sensitivity} that, when \(\lambda_3\) equals to 0.4 or 0.5, the performance of the model stays stable. In other words, the hyperparameter \(\lambda_2\) can be selected in a wide range, which means that the selection of hyperparameters in GraphAIR model does not impact the performance too much. Nevertheless, too small or too large hyperparameters \((\lambda_2, \lambda_3)\) greatly attenuate or accentuate the contribution of one branch, and thus lead to downgraded performance.

\section{Conclusion}
In this paper, we have firstly proved that existing mainstream GCN-based models have difficulty in well capturing the complicated non-linearity of graph data. Then, in order to better capture the complicated and non-linear distributions of nodes, we have proposed a novel GraphAIR framework that explicitly models the neighborhood interaction in addition to the neighborhood aggregation scheme. By employing residual learning strategy, we disentangle learning the neighborhood interaction from the neighborhood aggregation, which makes the optimization easier. The proposed GraphAIR is compatible with most existing graph convolutional models and it can provide a plug-and-play module for the neighborhood interaction. Finally, GraphAIR based on well-known models including GCN, GraphSAGE, and GAT have been thoroughly investigated through empirical evaluation. Extensive experiments on benchmark tasks including node classification and link prediction demonstrate the effectiveness of our model.

\section*{Acknowledgements}
This work is supported by National Natural Science Foundation of China (61772528) and National Key Research and Development Program (2018YFB1402600, 2016YFB1001000).

\bibliographystyle{IEEEtran}
\bibliography{pr}

\begin{IEEEbiography}[{\includegraphics[width=1in,height=1.25in,clip,keepaspectratio]{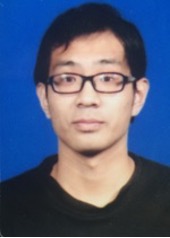}}]{Fenyu Hu}
received the B.S. degree in School of Electronical Engineering from Beijing University of Post and Communications, in 2017. He is currently pursuing Ph.D. degree in Center for Research on Intelligent Perception and Computing (CRIPAC) at National Laboratory of Pattern Recognition (NLPR), Institute of Automation, Chinese Academy of Sciences (CASIA), Beijing, China. His research interests include data mining, machine learning, recommender systems, and information retrieval.
\end{IEEEbiography}

\begin{IEEEbiography}[{\includegraphics[width=1in,height=1.25in,clip,keepaspectratio]{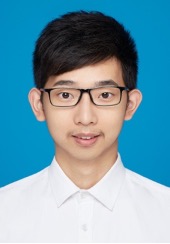}}]{Yanqiao Zhu}
is currently pursuing his master’s degree of Computer Science at Center for Research on Intelligent Perception and Computing (CRIPAC) at National Laboratory of Pattern Recognition (NLPR), Institute of Automation, Chinese Academy of Sciences (CASIA). Previously, he obtained his Bachelor of Engineering degree of Software Engineering from Tongji University in 2019. His current research interests mainly lie in the fields of machine learning (with an emphasis on graph representation learning) and its application to recommender systems.
\end{IEEEbiography}

\begin{IEEEbiography}[{\includegraphics[width=1in,height=1.25in,clip,keepaspectratio]{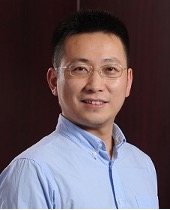}}]{Shu Wu}
received his B.S. degree from Hunan University, China, in 2004, M.S. degree from Xiamen University, China, in 2007, and Ph.D. degree from Department of Computer Science, University of Sherbrooke, Quebec, Canada, all in computer science. He is an Associate Professor with Center for Research on Intelligent Perception and Computing (CRIPAC) at National Laboratory of Pattern Recognition (NLPR), Institute of Automation, Chinese Academy of Sciences (CASIA). He has published more than 20 papers in the areas of data mining and information retrieval in international journals and conferences, such as IEEE TKDE, IEEE THMS, AAAI, ICDM, SIGIR, and CIKM. His research interests include data mining, information retrieval, and recommendation systems.
\end{IEEEbiography}

\begin{IEEEbiography}[{\includegraphics[width=1in,height=1.25in,clip,keepaspectratio]{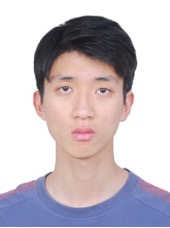}}]{Weiran Huang}
is currently a first-year Ph.D. student at The Chinese University of Hong Kong. Before that, he received his Bachelor of Engineering degree of Computer Science from Beijing University of Posts and Telecommunications. From May 2018 to May 2019, he was a research intern under the supervision of Dr. Shu Wu in Center for Research on Intelligent Perception and Computing (CRIPAC) at National Laboratory of Pattern Recognition (NLPR), Institute of Automation, Chinese Academy of Sciences (CASIA), Beijing China. His research interests include machine learning systems, data mining, and information retrieval.
\end{IEEEbiography}

\begin{IEEEbiography}[{\includegraphics[width=1in,height=1.25in,clip,keepaspectratio]{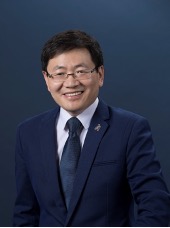}}]{Liang Wang}
received both the B.S. and M.S. degrees from Anhui University in 1997 and 2000, respectively, and the Ph.D. degree from the Institute of Automation, Chinese Academy of Sciences (CASIA) in 2004. From 2004 to 2010, he was a research assistant at Imperial College London, United Kingdom, and Monash University, Australia, a research fellow at the University of Melbourne, Australia, and a lecturer at the University of Bath, United Kingdom, respectively. Currently, he is a full professor of the Hundred Talents Program at the National Lab of Pattern Recognition, CASIA. His major research interests include machine learning, pattern recognition, and computer vision. He has widely published in highly ranked international journals such as IEEE TPAMI and IEEE TIP and leading international conferences such as CVPR, ICCV, and ICDM. He is an IEEE Fellow and an IAPR Fellow.
\end{IEEEbiography}

\begin{IEEEbiography}[{\includegraphics[width=1in,height=1.25in,clip,keepaspectratio]{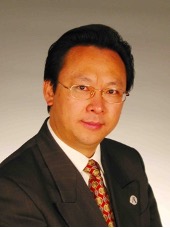}}]{Tieniu Tan}
received his BSc degree in electronic engineering from Xi'an Jiaotong University, China, in 1984, and his MSc and PhD degrees in electronic engineering from Imperial College London, U.K., in 1986 and 1989, respectively. In October 1989, he joined the Computational Vision Group at the Department of Computer Science, The University of Reading, Reading, U.K., where he worked as a Research Fellow, Senior Research Fellow and Lecturer. In January 1998, he returned to China to join the National Laboratory of Pattern Recognition (NLPR), Institute of Automation of the Chinese Academy of Sciences (CAS), Beijing, China, where he is currently Professor and former director (1998—2013) of the NLPR and Center for Research on Intelligent Perception and Computing (CRIPAC), and was Director General of the Institute (2000—2007). He was also Vice President of the Chinese Academy of Sciences (2015—2016). His current research interests include biometrics, image and video understanding, and information content security. Dr. Tan is a Fellow of CAS, TWAS (The World Academy of Sciences for the advancement of science in developing countries), IEEE and IAPR, and an International Fellow of the UK Royal Academy of Engineering. He is or has served as Associate Editor or member of editorial boards of many leading international journals including IEEE Transactions on Pattern Analysis and Machine Intelligence (PAMI), IEEE Transactions on Automation Science and Engineering, IEEE Transactions on Information Forensics and Security, IEEE Transactions on Circuits and Systems for Video Technology, Pattern Recognition, Pattern Recognition Letters, Image and Vision Computing, etc.
\end{IEEEbiography}

\end{document}